\documentclass[11pt]{article}

\usepackage[margin=1.12in]{geometry}
\usepackage{amsmath,amssymb,amsthm,mathtools}
\usepackage{microtype}
\usepackage{enumitem}
\usepackage{xcolor}
\usepackage{hyperref}
\usepackage{booktabs}
\usepackage{array}
\usepackage{listings}
\usepackage[T1]{fontenc}
\usepackage{lmodern}

\hypersetup{
  colorlinks=true,
  linkcolor=blue!55!black,
  citecolor=blue!55!black,
  urlcolor=blue!55!black
}

\lstdefinelanguage{Lean}{
  keywords={theorem,def,lemma,class,structure,where,by,exact,have,fun,intro,forall,exists,Prop,Type,Set,if,then,else,let,match,with,import,open,namespace,end,instance,noncomputable,abbrev},
  sensitive=true,
  comment=[l]{--},
  morecomment=[s]{/-}{-/},
  morestring=[b]"
}
\newtheorem{theorem}{Theorem}[section]
\newtheorem{proposition}[theorem]{Proposition}
\newtheorem{corollary}[theorem]{Corollary}

\newtheorem{definition}[theorem]{Definition}
\newtheorem{remark}[theorem]{Remark}
\newtheorem{openproblem}[theorem]{Open problem}

\newcommand{\X}{\mathcal X}
\newcommand{\B}{\mathcal B}
\newcommand{\C}{\mathcal C}
\newcommand{\R}{\mathbb R}
\newcommand{\N}{\mathbb N}

\newcommand{\one}{\mathbf 1}
\newcommand{\Emp}{\widehat L}

\newcommand{\Gap}{\Gamma}
\newcommand{\WB}{\mathrm{WB}}
\newcommand{\WBm}{\mathrm{WB}_{\mathrm{meas}}}
\newcommand{\KW}{\mathrm{KW}}

\title{Null Measurability at the Symmetrization Interface in VC Learning}
\author{Dhruv Gupta\\AI and Robotics Technology Park, Indian Institute of Science, Bengaluru\\\texttt{dhruvgupta@iisc.ac.in}}
\date{April 2026}

\begin{document}
\maketitle

\begin{abstract}
Recent work revisiting measurability in the fundamental theorem of statistical learning imposes Borel measurability of ghost-gap suprema.  We show that, at the one-sided ghost-gap interface actually used by the standard symmetrization proof, this requirement is stronger than necessary.  For any Borel-parameterized concept class on a Polish domain, the bad event ``there exists a hypothesis whose ghost empirical error exceeds its training empirical error by at least $\varepsilon/2$'' is analytic.  By Choquet capacitability, it is therefore measurable in the completion of every finite Borel measure.  We then construct a concept class whose bad event is null-measurable but not Borel, giving a strict separation from the Borel supremum condition.  Finally, we prove closure under patching, fixed and countable interpolation, and fiber-product amalgamation, showing that the weaker regularity level is stable under natural concept-class constructors.  In the realizable setting, where targets belong to the class and are measurable, these results weaken the measurability hypothesis needed by the symmetrization route from finite VC dimension to PAC learnability.  The main results and the descriptive-set-theoretic infrastructure used by them are formalized in Lean 4.
\end{abstract}

\section{Introduction}

The fundamental theorem of statistical learning identifies finite VC dimension as the exact combinatorial criterion for distribution-free PAC learnability in binary classification.  The classical implication from finite VC dimension to learnability proceeds through double sampling and symmetrization: one controls a bad event on pairs of samples and then bounds its probability using the growth function and Sauer--Shelah estimates \cite{VapnikChervonenkis1971,BEHW1989,Sauer1972,Shelah1972,SSBD2014}.  The proof is elementary in its combinatorial core, but it hides a measurability issue.  The bad event is defined by an existential quantifier over an often uncountable concept class, and an uncountable union of measurable sets need not be Borel.

Measurability problems of this kind are classical in empirical process theory.  Pollard's permissibility condition, Dudley's image-admissible Suslin classes, and the framework of van der Vaart and Wellner all record ways to keep empirical-process suprema and outer-probability arguments mathematically meaningful \cite{Pollard1984,Dudley1984,VdVW1996,VdVW2000}.  Recent work of Krapp and Wirth revisits the fundamental theorem from this perspective and extracts explicit measurability assumptions for the learning-theoretic proof, with further applications to tame model theory \cite{KrappWirth2024}.  Their intervention is important: it turns a tacit side condition in textbook proofs into a mathematical object.

This paper makes a more local point.  The one-sided ghost-gap event actually used by the standard symmetrization argument does not need to be Borel.  It is enough that the event be measurable in the completion of the relevant product probability measure.  In Lean~4 terminology~\cite{Lean4}, the correct event-level hypothesis is \texttt{NullMeasurableSet}, not \texttt{MeasurableSet}.  This weakening is not merely cosmetic.  It admits Borel-parameterized classes whose ghost bad events are analytic but not Borel, and it is stable under the same constructor operations that naturally build larger classes from smaller ones.

The mechanism is simple.  Suppose that $\Theta$ is a standard Borel space and that
\[
  e : \Theta \times \X \to \{0,1\}
\]
is jointly measurable.  For $\theta\in\Theta$, write $e_\theta(x)=e(\theta,x)$ and let $\C=\{e_\theta: \theta\in\Theta\}$.  For a target $c$, sample size $m$, and threshold $\varepsilon$, the witness set
\[
  W_{e,c,m,\varepsilon}
  = \{(\theta,S,T): \widehat L_T(e_\theta,c)-\widehat L_S(e_\theta,c)\ge \varepsilon/2\}
\]
is Borel because it is a superlevel set of a finite sum of measurable functions.  Its projection to $(S,T)$ is the bad event.  Suslin's theorem says that this projection is analytic, and Choquet capacitability implies that analytic sets in Polish spaces are universally measurable.  Thus the bad event is null-measurable under every finite Borel measure.  The proof path is
\begin{align*}
  \text{jointly measurable evaluator}
  &\Longrightarrow \text{Borel witness graph} \\
  &\Longrightarrow \text{analytic bad event}
  \Longrightarrow \text{null-measurable bad event}.
\end{align*}

The descriptive-set-theoretic step is standard, but its placement is the point.  Earlier empirical-process regularity conditions package this phenomenon globally through permissibility or image-admissible Suslin hypotheses.  Krapp and Wirth package it through Borel measurability of gap maps.  We isolate the precise event-level regularity used by the VC symmetrization proof and formalize it as an interface.  This also clarifies what is, and is not, a new mathematical claim in this paper.  The facts that analytic sets are universally measurable and that Borel projections are analytic are classical \cite{Suslin1917,Choquet1954,Kechris1995}.  The contribution here is the localization of those facts at the one-sided symmetrization interface, the strict concept-class witness against the Borel condition, the closure algebra for Borel-parameterized constructors, and the Lean formalization certifying the translation.

Our main contributions are as follows.
\begin{enumerate}[label=(\arabic*)]
\item We prove a Borel--analytic bridge theorem for one-sided ghost-gap bad events of Borel-parameterized concept classes.
\item We construct a strict separation witness: a singleton concept class over an analytic non-Borel set whose one-sided ghost bad event is analytic and null-measurable but not Borel.
\item We prove closure under patching, fixed and countable interpolation, and fiber-product amalgamation.
\item We spell out the resulting weakened measurability interface for the realizable symmetrization route from finite VC dimension to PAC learnability.
\item We give a Lean 4 formalization of the main results, including the Choquet-capacitability route from analytic sets to null-measurability~\cite{FLTKernel2026}.\footnote{Code and Lean kernel available at \url{https://github.com/Zetetic-Dhruv/formal-learning-theory-kernel/tree/v3.3.0-paper}.}
\end{enumerate}

Two boundaries are important.  First, our positive bridge is stated for measurable targets.  This is the natural interface for realizable PAC learning, where the target belongs to the concept class and hypotheses are measurable.  Extending the bridge to arbitrary nonmeasurable targets would require a different argument.  Second, we do not claim to reprove every equivalent form of the fundamental theorem.  In particular, the sample-compression direction associated with Moran and Yehudayoff is orthogonal to the measurability interface studied here \cite{MoranYehudayoff2016}.

\section{Setting and regularity notions}

Let $(\X,\B)$ be a measurable space.  A binary concept is a function $h:\X\to\{0,1\}$, and a concept class is a set $\C\subseteq\{0,1\}^\X$.  For a sample $S=(x_1,\ldots,x_m)\in\X^m$ and target concept $c:\X\to\{0,1\}$, define the empirical zero-one error
\[
  \Emp_S(h,c)=\frac{1}{m}\sum_{i=1}^m \one[h(x_i)\ne c(x_i)]
\]
for $m>0$, and set the empty-sample empirical average to $0$ when convenient.  For a pair of samples $p=(S,T)\in \X^m\times \X^m$, define the one-sided ghost gap
\[
  \Gap^c_h(p)=\Emp_T(h,c)-\Emp_S(h,c).
\]
The corresponding one-sided bad event is
\[
  E_{\C,c,m,\varepsilon}
  =\{p\in \X^m\times \X^m: \exists h\in\C\; \Gap^c_h(p)\ge \varepsilon/2\}.
\]

\begin{definition}[Event-level null-measurability]\label{def:WB}
Let $\mu$ be a probability measure on $\X$, and let $\mu^{2m}$ denote the product measure on $\X^m\times \X^m$.
\begin{enumerate}[label=(\roman*)]
\item $\WB(\C)$ holds if, for every probability measure $\mu$, every target $c:\X\to\{0,1\}$, every $m\in\N$, and every $\varepsilon>0$, the event $E_{\C,c,m,\varepsilon}$ is measurable in the completion of $\mu^{2m}$.
\item $\WBm(\C)$ holds if the same condition is required only for measurable targets $c$.
\end{enumerate}
\end{definition}

In the Lean development, the unrestricted condition is \texttt{WellBehavedVC}; the measurable-target variant is \texttt{WellBehavedVCMeasTarget}.  The distinction matters only when one quantifies over arbitrary targets.  In the realizable setting, if $c\in\C$ and all concepts in $\C$ are measurable, then $c$ is measurable and $\WBm$ is the relevant interface.

We compare this event-level condition to a Borel-level supremum condition.  Define
\[
  G_{\C,c,m}(p)=\sup_{h\in\C} \Gap^c_h(p),
\]
where the supremum is understood over the finite grid of possible empirical-gap values.  The Lean development uses the finite-grid representation to avoid relying on non-attained suprema.

\begin{definition}[Borel ghost-gap condition]
We write $\KW(\C)$ if for every target $c$ and every $m\in\N$, the map $p\mapsto G_{\C,c,m}(p)$ is Borel measurable.
\end{definition}

This condition is intentionally close to the one-sided part of the Krapp--Wirth interface, though their paper also studies absolute and uniform-convergence gap maps.  If $G_{\C,c,m}$ is Borel measurable, then $E_{\C,c,m,\varepsilon}=G_{\C,c,m}^{-1}([\varepsilon/2,\infty))$ is Borel.  The converse need not hold, and the symmetrization argument only needs the event to be integrable as an indicator against the relevant product measure.

\begin{proposition}[Realizable interface]\label{prop:realizable-interface}
Suppose the realizable PAC argument only quantifies over targets $c\in\C$, and every $h\in\C$ is measurable.  Then $\WBm(\C)$ supplies the measurability needed by the one-sided symmetrization proof.  Borel measurability of $G_{\C,c,m}$ is not used.
\end{proposition}

\begin{proof}
The measurable event inserted into the double-sample proof is exactly $E_{\C,c,m,\varepsilon}$ for a target $c\in\C$.  Under the hypothesis, such $c$ is measurable.  The proof integrates the indicator of this event and uses finite sums, exchangeability, and monotonicity.  These steps require the set to be measurable in the completion of the product measure, since indicators of null-measurable sets are a.e. measurable and their lower Lebesgue integrals have the expected event probabilities.  No step requires a Borel measurable supremum map.
\end{proof}

\section{The Borel--analytic bridge}

Assume now that $\X$ is a Polish space with its Borel $\sigma$-algebra.  Let $\Theta$ be a standard Borel space and let
\[
  e:\Theta\times\X\to\{0,1\}
\]
be jointly measurable.  Write $e_\theta(x)=e(\theta,x)$ and $\C_e=\{e_\theta:\theta\in\Theta\}$.

For $m\in\N$ and $\varepsilon\in\R$, define the witness set
\[
  W_{e,c,m,\varepsilon}
  =\{(\theta,p)\in\Theta\times(\X^m\times\X^m):
      \Gap^c_{e_\theta}(p)\ge\varepsilon/2\},
\]
and its projection
\[
  E_{e,c,m,\varepsilon}
  =\{p: \exists \theta\in\Theta\ (\theta,p)\in W_{e,c,m,\varepsilon}\}.
\]

\begin{theorem}[Borel--analytic bridge]\label{thm:borel-analytic-bridge}
Let $\X$ be Polish, $\Theta$ standard Borel, and $e:\Theta\times\X\to\{0,1\}$ jointly measurable.  For every measurable target $c$, sample size $m$, and threshold $\varepsilon$:
\begin{enumerate}[label=(\alph*)]
\item $W_{e,c,m,\varepsilon}$ is Borel in $\Theta\times(\X^m\times\X^m)$;
\item $E_{e,c,m,\varepsilon}$ is analytic in $\X^m\times\X^m$;
\item for every finite Borel measure $\nu$ on $\X^m\times\X^m$, $E_{e,c,m,\varepsilon}$ is measurable in the completion of $\nu$.
\end{enumerate}
Consequently $\C_e$ satisfies $\WBm$.
\end{theorem}

\begin{proof}
For (a), consider the map
\[
  (\theta,S,T)\mapsto \Emp_T(e_\theta,c)-\Emp_S(e_\theta,c).
\]
Each summand has the form
\[
  (\theta,S,T)\mapsto \one[e(\theta,T_i)\ne c(T_i)]
\]
or the corresponding term with $S_i$.  The coordinate projections to $T_i$ and $S_i$ are measurable, $e$ and $c$ are measurable, equality in the two-point measurable space is measurable, and finite sums preserve measurability.  Thus the ghost-gap map is measurable as a real-valued function.  The witness set is its preimage of $[\varepsilon/2,\infty)$, hence Borel.

For (b), the sample space $\X^m\times\X^m$ is again Polish, and $\Theta\times(\X^m\times\X^m)$ is a standard Borel space.  The bad event is the image of the Borel set $W_{e,c,m,\varepsilon}$ under the measurable projection $(\theta,p)\mapsto p$.  Suslin's projection theorem gives analyticity.

For (c), analytic subsets of Polish spaces are universally measurable.  One standard proof is via Choquet capacitability: analytic sets are capacitable for every finite Borel measure, hence admit Borel supersets whose excess measure is arbitrarily small, and therefore are measurable in the completion.  Applying this to $E_{e,c,m,\varepsilon}$ under the finite measure $\nu$ proves null-measurability.  Taking $\nu$ to be the product measure generated by any probability measure on $\X$ gives $\WBm(\C_e)$.
\end{proof}

\begin{remark}[Relation to earlier regularity notions]
Theorem~\ref{thm:borel-analytic-bridge} is a local event-level version of classical empirical-process regularity.  It is closest in spirit to image-admissible Suslin and permissibility assumptions, and to preservation theorems for Glivenko--Cantelli classes under measurable composition.  The difference is that we track the descriptive-set-theoretic stratum reached by the specific one-sided ghost event used in the symmetrization proof.
\end{remark}

\section{A strict separation from Borel measurability}

The previous theorem shows that null-measurability is sufficient for Borel-parameterized classes.  We now show that it is genuinely weaker than Borel measurability of the bad event.

Let $A\subseteq\R$ be analytic but non-Borel.  Such sets exist in ZFC in every uncountable Polish space.  Define
\[
  \C_A=\{0\}\cup\{\one_{\{a\}}: a\in A\}.
\]
Each hypothesis is Borel measurable, since singletons in $\R$ are closed.

\begin{theorem}[Strict separation]\label{thm:strict-separation}
Assume $A\subseteq\R$ is analytic and non-Borel.  Then the class $\C_A$ satisfies $\WBm$, but it does not satisfy $\KW$.
\end{theorem}

\begin{proof}
Because $A$ is analytic, choose a standard Borel parameter space $\Theta$ and a measurable map $g:\Theta\to\R$ with range $A$.  Parameterize $\C_A$ by $\{0,1\}\times\Theta$ via
\[
  e(b,\theta)(x)=\begin{cases}
  \one[x=g(\theta)] & b=1,\\
  0 & b=0.
  \end{cases}
\]
The evaluation map $((b,\theta),x)\mapsto e(b,\theta)(x)$ is measurable: the singleton equality relation $x=g(\theta)$ is Borel in $\Theta\times\R$, and the two branches are constant or singleton indicators.  Hence Theorem~\ref{thm:borel-analytic-bridge} gives $\WBm(\C_A)$.

It remains to show that the Borel condition fails.  Take $m=1$, target $c\equiv0$, and threshold $\varepsilon=1$.  Let
\[
  P_A=\{(x,y)\in\R^2: y\in A\text{ and }x\ne y\}.
\]
For a pair $(x,y)$ of training and ghost points, the hypothesis $\one_{\{a\}}$ has ghost error minus training error equal to $1$ exactly when $y=a$ and $x\ne a$.  Therefore the bad event for $\C_A$ at this $m,c,\varepsilon$ is precisely the preimage of $P_A$ under the map from the one-point training/ghost sample pair to its two coordinates.

The set $P_A$ is analytic because it is $(\R\times A)\cap\{(x,y):x\ne y\}$, the intersection of an analytic set with a Borel set.  It is not Borel.  Indeed, since $A$ is non-Borel, choose $a_0\notin A$.  If $P_A$ were Borel, then the measurable section map $y\mapsto(a_0,y)$ would pull it back to $A$, contradicting non-Borelicity.  Thus the corresponding bad event is analytic but non-Borel.  If the one-sided supremum map $G_{\C_A,c,1}$ were Borel, then its superlevel set at $1/2$ would be Borel, contradiction.  Hence $\KW(\C_A)$ fails.
\end{proof}

\begin{remark}
This witness is related in spirit to the Dudley--Durst~\cite{DudleyDurst1981} and Pestov~\cite{Pestov2011} discussions of measurability pathologies in VC theory, but the role is different.  The present witness targets the Borel/non-Borel gap of the one-sided symmetrization bad event.  It is not a new set-theoretic construction of an analytic non-Borel set; it is a learning-theoretic placement of that classical descriptive-set-theoretic object.
\end{remark}

\section{Closure under concept-class constructors}

A single separation witness could be dismissed as pathology.  The more important point is closure: the null-measurable endpoint is stable under natural ways of forming new concept classes from old ones.

\subsection{Patching}

Let $e_1:\Theta_1\times\X\to\{0,1\}$, $e_2:\Theta_2\times\X\to\{0,1\}$, and $r:P\times\X\to\{0,1\}$ be jointly measurable maps with standard Borel parameter spaces.  Define
\[
  \mathrm{Patch}(e_1,e_2,r)
  =\left\{x\mapsto
    \begin{cases}
    e_1(\theta_1,x),& r(\rho,x)=1,\\
    e_2(\theta_2,x),& r(\rho,x)=0,
    \end{cases}
    : \theta_1\in\Theta_1,\theta_2\in\Theta_2,\rho\in P\right\}.
\]

\begin{theorem}[Patching preserves $\WBm$]\label{thm:patching}
Under the assumptions above, $\mathrm{Patch}(e_1,e_2,r)$ satisfies $\WBm$.
\end{theorem}

\begin{proof}
Use the product parameter space $\Theta_1\times\Theta_2\times P$.  The patched evaluator
\[
  ((\theta_1,\theta_2,\rho),x)
  \mapsto
  \begin{cases}
  e_1(\theta_1,x),& r(\rho,x)=1,\\
  e_2(\theta_2,x),& r(\rho,x)=0
  \end{cases}
\]
is jointly measurable by the measurable piecewise-map theorem: the routing set $\{(\rho,x):r(\rho,x)=1\}$ is Borel, and both branches are jointly measurable after composition with coordinate projections.  Apply Theorem~\ref{thm:borel-analytic-bridge}.
\end{proof}

\subsection{Interpolation}

For a fixed measurable set $A\subseteq\X$, the fixed-region interpolation of two classes is
\[
  \mathrm{Interp}_A(\C_1,\C_2)=\{x\mapsto h_1(x)\text{ on }A,
    \;h_2(x)\text{ on }A^c: h_1\in\C_1,h_2\in\C_2\}.
\]
For a countable measurable family $(A_n)_{n\in\N}$, define the countable-router interpolation by allowing the router to choose one of the $A_n$.

\begin{corollary}[Fixed and countable interpolation]\label{cor:interpolation}
If $\C_1$ and $\C_2$ are Borel-parameterized by jointly measurable evaluators over standard Borel parameter spaces, then every fixed-region interpolation and every countable measurable-family interpolation of $\C_1$ and $\C_2$ satisfies $\WBm$.
\end{corollary}

\begin{proof}
A fixed measurable region is the special case of Theorem~\ref{thm:patching} with one router parameter.  A countable measurable family is the special case with router parameter $n\in\N$ and routing set $A_n$.  The map $(n,x)\mapsto \one[x\in A_n]$ is measurable because its graph is the countable union $\bigcup_n \{n\}\times A_n$.  The conclusion follows from Theorem~\ref{thm:patching}.
\end{proof}

\subsection{Fiber-product amalgamation}

Let $\pi_1:\Theta_1\to S$ and $\pi_2:\Theta_2\to S$ be measurable maps into a standard Borel space $S$, and let
\[
  m : \Theta_1\times\Theta_2\times\X\to\{0,1\}
\]
be jointly measurable.  The fiber-product amalgamation class is
\[
  \mathrm{Amal}(\pi_1,\pi_2,m)
  =\{x\mapsto m(\theta_1,\theta_2,x): \pi_1(\theta_1)=\pi_2(\theta_2)\}.
\]

\begin{theorem}[Amalgamation preserves $\WBm$]\label{thm:amalgamation}
Suppose $\Theta_1,\Theta_2,S$ are standard Borel, $\pi_1,\pi_2$ are measurable, and $m$ is jointly measurable.  Then $\mathrm{Amal}(\pi_1,\pi_2,m)$ satisfies $\WBm$.
\end{theorem}

\begin{proof}
The agreement relation
\[
  R=\{(\theta_1,\theta_2):\pi_1(\theta_1)=\pi_2(\theta_2)\}
\]
is Borel.  Indeed, it is the preimage of the diagonal in $S\times S$ under $(\theta_1,\theta_2)\mapsto(\pi_1(\theta_1),\pi_2(\theta_2))$, and the diagonal of a standard Borel space is measurable.  The measurable subtype $R$ is again a standard Borel space.  Restrict the merge evaluator to $R\times\X$; joint measurability is preserved under the measurable inclusion $R\hookrightarrow\Theta_1\times\Theta_2$.  The result follows from Theorem~\ref{thm:borel-analytic-bridge}.
\end{proof}

\begin{remark}[Closure target]
Theorems~\ref{thm:patching}--\ref{thm:amalgamation} do not say that Borel measurability of the bad event is preserved.  They identify the stable positive target: after existential elimination of parameters, analytic/null-measurable is what one should expect.  The strict witness in Theorem~\ref{thm:strict-separation} shows that forcing all such events back into the Borel $\sigma$-algebra is not generally possible.
\end{remark}

\section{Consequence for the symmetrization route}

The symmetrization proof of the implication
\[
  \operatorname{VCdim}(\C)<\infty \quad\Longrightarrow\quad \C\text{ is PAC learnable}
\]
controls the probability of the one-sided ghost event $E_{\C,c,m,\varepsilon}$.  The event-level regularity used by this argument is exactly completion-measurability of $E_{\C,c,m,\varepsilon}$ under the relevant product measure.

\begin{corollary}[Weakened realizable measurability interface]\label{cor:weakened-interface}
Let $\X$ be Polish and let $\C\subseteq\{0,1\}^\X$ be a concept class of measurable hypotheses.  In the realizable setting, the standard symmetrization route from finite VC dimension to PAC learnability requires $\WBm(\C)$, not the stronger Borel condition $\KW(\C)$.  In particular, every Borel-parameterized class, and every class obtained from such classes by the patching, fixed/countable interpolation, and fiber-product amalgamation operations above, satisfies the event-level measurability needed by that route.
\end{corollary}

\begin{proof}
By Proposition~\ref{prop:realizable-interface}, the target $c$ appearing in the realizable argument is measurable and the symmetrization proof only needs null-measurability of the one-sided ghost bad event.  Theorem~\ref{thm:borel-analytic-bridge} supplies this condition for Borel-parameterized classes, and Theorems~\ref{thm:patching}--\ref{thm:amalgamation} supply it for the closure constructions.  The probability bound itself is the standard VC symmetrization bound once the event is legitimate as an integrable indicator.
\end{proof}

\section{Discussion and open problems}

The paper leaves several seams deliberately visible.

\paragraph{Measurable targets.}
The positive bridge is naturally stated for measurable targets.  This is the right theorem for the realizable setting.  Extending the bridge to arbitrary targets would require either a different coding of targets or a regularity condition on the ambient space strong enough to make all Boolean concepts measurable.

\paragraph{One-sided gaps.}
The standard double-sample symmetrization proof uses a one-sided event.  The present analysis is therefore organized around one-sided ghost gaps.  Absolute-gap versions can be studied, but the strict separation already appears at the one-sided interface.

\paragraph{Relationship to Krapp--Wirth and empirical-process regularity.}
Krapp and Wirth's condition is a sufficient Borel-level condition; the present paper shows that the event-level completion condition suffices at the symmetrization interface and that the gap between the two is strict.  A complementary observation on the model-theoretic side is due to Krapp, Vermeil, and Wirth~\cite{KrappVermeilWirth2025}, who construct a first-order definable non-Borel set in a subfield of $\R$ with the independence property: even strong syntactic regularity assumptions on the indexing structure cannot substitute for explicit measurability hypotheses.  Their construction is the model-theoretic analogue of the strict separation we exhibit on the learning-theoretic side; both indicate that the right hypothesis lives at the descriptive-set-theoretic level rather than at the level of definability or Borel measurability of suprema.  Pollard/Dudley-style admissibility conditions already imply forms of universal measurability; our contribution is the modern explicit event-level formulation and the Lean-checked placement in the finite-VC symmetrization chain.  Parallel work on universal-learning frameworks~\cite{BousquetEtAl2021} extends FTSL-type characterizations along an orthogonal axis.

\paragraph{Formalization as a forcing device.}
The Lean development did not merely check a finished proof.  It forced the regularity theorem to be stated at the right level.  Once the bad event was only null-measurable, the Tonelli/indicator part of the proof had to be refactored through null-measurable indicator lemmas rather than Borel-set indicator lemmas.  This is the formal reason the distinction is not cosmetic.

\begin{openproblem}[Arbitrary targets]
Can the measurable-target bridge be extended to arbitrary targets without reintroducing a Borel supremum requirement?
\end{openproblem}

\begin{openproblem}[Measurability complexity]
Is there a useful quantitative invariant measuring the descriptive complexity of bad events under concept-class constructors, analogous to VC dimension on the combinatorial side?
\end{openproblem}

\begin{openproblem}[Multiclass labels]
What is the correct multiclass or real-valued analogue of the Borel--analytic bridge?  The finite-label case should follow the same finite-sum logic, but real-valued prediction and loss functions require a separate regularity analysis.
\end{openproblem}

\begin{openproblem}[Maximal closure algebra]
Characterize the largest natural algebra of concept-class constructors preserving $\WBm$.  The patching, countable interpolation, and fiber-product amalgamation theorems are lower bounds on this algebra, not a classification.
\end{openproblem}

\section*{Acknowledgments}
The author thanks the AI and Robotics Technology Park at the Indian Institute of Science for institutional support.

\appendix

\section{Lean artifact map}

The mathematical statements above are formalized in Lean~4 in the accompanying formal learning theory kernel~\cite{FLTKernel2026}, snapshotted at tag \texttt{v3.3.0-paper} (commit \texttt{e48c0f6}).  Table~\ref{tab:lean-map} records the theorem names corresponding to the paper statements.  Verification is via the comparator~\cite{Comparator}; see Appendix~\ref{sec:verification-standard}.

\begin{table}[h]
\centering
\scriptsize
\begin{tabular}{>{\raggedright\arraybackslash}p{0.20\textwidth}>{\raggedright\arraybackslash}p{0.39\textwidth}>{\raggedright\arraybackslash}p{0.33\textwidth}}
\toprule
Mathematical content & Lean name & File \\
\midrule
Analytic sets are null-measurable & \texttt{analyticSet\_nullMeasurableSet} & \texttt{PureMath/AnalyticMeasurability.lean} \\
Choquet compact capacity for analytic sets & \texttt{AnalyticSet.compactCap\_eq} & \texttt{PureMath/ChoquetCapacity.lean} \\
Borel witness graph & \texttt{paramWitnessSet\_measurable} & \texttt{Complexity/BorelAnalyticBridge.lean} \\
Analytic bad event & \texttt{borel\_param\_badEvent\_analytic} & \texttt{Complexity/BorelAnalyticBridge.lean} \\
Positive bridge & \texttt{borel\_param\_wellBehavedVCMeasTarget} & \texttt{Complexity/BorelAnalyticBridge.lean} \\
Patching closure & \texttt{patch\_borel\_param\_}\newline\texttt{wellBehavedVCMeasTarget} & \texttt{Complexity/BorelAnalyticBridge.lean} \\
Fixed interpolation closure & \texttt{interpClassFixed\_wellBehaved} & \texttt{Complexity/Interpolation.lean} \\
Countable interpolation closure & \texttt{interpClassCountable\_wellBehaved} & \texttt{Complexity/Interpolation.lean} \\
Amalgamation closure & \texttt{amalgClass\_wellBehaved} & \texttt{Complexity/Amalgamation.lean} \\
Strict separation witness & \texttt{analytic\_nonborel\_set\_}\newline\texttt{gives\_measTarget\_separation} & \texttt{Theorem/BorelAnalytic\-Separation.lean} \\
One-sided regularity layer & \texttt{WellBehavedVC}, \texttt{WellBehavedVCMeasTarget}, \texttt{KrappWirthWellBehaved} & \texttt{Complexity/Measurability.lean} \\
\bottomrule
\end{tabular}
\caption{Lean declarations corresponding to the main mathematical results.}
\label{tab:lean-map}
\end{table}

\section{Selected Lean statements}

This appendix gives representative Lean statements for the key formalized results.  The code below is excerpted for readability; the full proofs are in the repository files listed in Table~\ref{tab:lean-map}.

\begin{lstlisting}
theorem analyticSet_nullMeasurableSet
    {alpha : Type*}
    [TopologicalSpace alpha] [MeasurableSpace alpha]
    [BorelSpace alpha] [PolishSpace alpha]
    {mu : MeasureTheory.Measure alpha} [MeasureTheory.IsFiniteMeasure mu]
    {s : Set alpha} (hs : MeasureTheory.AnalyticSet s) :
    MeasureTheory.NullMeasurableSet s mu
\end{lstlisting}

\begin{lstlisting}
theorem paramWitnessSet_measurable
    {X : Type u} [MeasurableSpace X]
    {Theta : Type*} [MeasurableSpace Theta]
    (e : Theta -> Concept X Bool)
    (he : Measurable (fun p : Theta x X => e p.1 p.2))
    (c : Concept X Bool) (hc : Measurable c)
    (m : Nat) (eps : Real) :
    MeasurableSet (paramWitnessSet e c m eps)
\end{lstlisting}

\begin{lstlisting}
theorem borel_param_badEvent_analytic
    {X : Type u} [MeasurableSpace X] [TopologicalSpace X]
    [PolishSpace X] [BorelSpace X]
    {Theta : Type*} [MeasurableSpace Theta] [StandardBorelSpace Theta]
    (e : Theta -> Concept X Bool)
    (he : Measurable (fun p : Theta x X => e p.1 p.2))
    (c : Concept X Bool) (hc : Measurable c)
    (m : Nat) (eps : Real) :
    MeasureTheory.AnalyticSet (paramBadEvent e c m eps)
\end{lstlisting}

\begin{lstlisting}
theorem borel_param_wellBehavedVCMeasTarget
    {X : Type u} [MeasurableSpace X] [TopologicalSpace X]
    [PolishSpace X] [BorelSpace X]
    {Theta : Type*} [MeasurableSpace Theta] [StandardBorelSpace Theta]
    (e : Theta -> Concept X Bool)
    (he : Measurable (fun p : Theta x X => e p.1 p.2)) :
    WellBehavedVCMeasTarget X (Set.range e)
\end{lstlisting}

\begin{lstlisting}
theorem patch_borel_param_wellBehavedVCMeasTarget
    {X : Type u} [MeasurableSpace X] [TopologicalSpace X]
    [PolishSpace X] [BorelSpace X]
    {Theta1 Theta2 Rho : Type*}
    [MeasurableSpace Theta1] [StandardBorelSpace Theta1]
    [MeasurableSpace Theta2] [StandardBorelSpace Theta2]
    [MeasurableSpace Rho] [StandardBorelSpace Rho]
    (e1 : Theta1 -> Concept X Bool)
    (e2 : Theta2 -> Concept X Bool)
    (r : Rho -> Concept X Bool)
    (he1 : Measurable (fun p : Theta1 x X => e1 p.1 p.2))
    (he2 : Measurable (fun p : Theta2 x X => e2 p.1 p.2))
    (hr : Measurable (fun p : Rho x X => r p.1 p.2)) :
    WellBehavedVCMeasTarget X (Set.range (patchEval e1 e2 r))
\end{lstlisting}

\begin{lstlisting}
theorem amalgClass_wellBehaved
    {X : Type u} [MeasurableSpace X] [TopologicalSpace X]
    [PolishSpace X] [BorelSpace X]
    {Theta1 Theta2 S : Type*}
    [MeasurableSpace Theta1] [StandardBorelSpace Theta1]
    [MeasurableSpace Theta2] [StandardBorelSpace Theta2]
    [MeasurableSpace S] [StandardBorelSpace S]
    (pi1 : Theta1 -> S) (pi2 : Theta2 -> S)
    (hpi1 : Measurable pi1) (hpi2 : Measurable pi2)
    (merge : Theta1 x Theta2 -> Concept X Bool)
    (hmerge : Measurable (fun p : (Theta1 x Theta2) x X =>
      merge p.1 p.2)) :
    WellBehavedVCMeasTarget X (amalgClass pi1 pi2 merge)
\end{lstlisting}

\begin{lstlisting}
theorem analytic_nonborel_set_gives_measTarget_separation
    (A : Set Real)
    (hA_an : MeasureTheory.AnalyticSet A)
    (hA_non : not MeasurableSet A) :
    KrappWirthSeparationMeasTarget
\end{lstlisting}

\section{Verification standard}\label{sec:verification-standard}

The kernel and these results are verified with the comparator~\cite{Comparator}, which is the strictest standard used in the Lean~4 community to certify formalizations developed with AI assistance.

\bibliographystyle{plain}
\bibliography{refs}

\end{document}